\documentclass[11pt,twoside]{article}

\usepackage[utf8]{inputenc}
\usepackage[T1]{fontenc}

\usepackage{geometry}
\usepackage{fancyhdr}
\usepackage{titlesec}

\usepackage{amsmath,amssymb,amsfonts,amsthm,mathtools,bm}

\usepackage{graphicx}
\usepackage{booktabs}
\usepackage{longtable}
\usepackage{array}
\usepackage{float}
\usepackage{placeins}

\usepackage{enumitem}
\usepackage[ruled,vlined,linesnumbered]{algorithm2e}

\usepackage{natbib}
\usepackage{orcidlink}
\usepackage{hyperref}

\setcitestyle{authoryear}
\hypersetup{
    colorlinks=true,
    linkcolor=blue,
    filecolor=blue,
    citecolor=blue,
    urlcolor=black
}

\fancypagestyle{plain}{
    \fancyhf{}
    
}

\newtheorem{theorem}{Theorem}[section]
\newtheorem{proposition}[theorem]{Proposition}

\newtheorem{assumption}[theorem]{Assumption}
\newtheorem{remark}[theorem]{Remark}

\SetKwInput{KwInput}{Input}
\SetKwInput{KwOutput}{Output}

\title{
\vspace{-1.5em}
\hrule height 1.5pt
\vspace{0.8em}
StrTransformer: Source-Wise Structured Transformers for Unsupervised Blind Source Recovery
\vspace{0.8em}
\hrule height 1.5pt
\vspace{1em}
}

\author{%
\begin{minipage}[t]{.78\textwidth}\centering\small
  \textbf{Yuan-Hao Wei\orcidlink{0000-0001-9439-0780}}\\
  \texttt{Yuan-Hao.Wei@outlook.com}\\
    \texttt{yuan-hao.wei@connect.polyu.hk}\\
    \texttt{yuanhao.wei1993@gmail.com}\\
\end{minipage}
}

\date{}

\begin{document}
\maketitle
\thispagestyle{plain}

\begin{abstract}
This paper proposes StrTransformer, a source-wise structured Transformer framework for blind source recovery and branch-wise latent modeling. Instead of using an encoder to infer latent variables, StrTransformer directly optimizes the latent source matrix together with an observation-space mixer and source-wise structural Transformer branches. The mixer enforces reconstruction consistency, while each Transformer branch imposes a differentiable structural constraint on one latent source trajectory. Specifically, each source is converted into multi-scale patch tokens, randomly masked, processed by a locality-biased Transformer, and evaluated through a masked patch reconstruction energy. This energy acts as an implicit source-wise structural prior. To encourage different latent branches to specialize into different temporal regimes, StrTransformer further introduces an ordered multi-scale controller that learns branch-specific patch-scale weights, ordered scale centers, and locality attention slopes. The resulting objective combines observation reconstruction, source-wise structural regularization, and modular auxiliary penalties for separation and scale specialization. We analyze the decoupling and coupling structure of the objective, the regularized exact-reconstruction fiber, and the reduction of permutation symmetry induced by ordered branch descriptors. A controlled case study shows that the learned branches converge to distinct temporal-scale structures and recover source-aligned latent trajectories under post-hoc evaluation. 
\end{abstract}


\section{Introduction}
\label{sec:introduction}

Blind source separation (BSS) aims to recover latent source signals from observed mixtures without direct access to the sources or the mixing mechanism. A classical and influential formulation is independent component analysis (ICA), where the latent components are assumed to be statistically independent and the observations are often modeled as linear mixtures. Under suitable non-Gaussianity or temporal-dependence assumptions, ICA and related second-order blind identification methods provide principled routes to source recovery \citep{comon1994ica,cardoso1993jade,belouchrani1997sobi,hyvarinen2000ica}. These methods have played a central role in signal processing, neuroscience, audio processing, and multivariate data analysis. However, their assumptions also reveal a limitation: when the mixing process is nonlinear, when the source structures are more complex than simple marginal non-Gaussianity, or when the desired representation is learned through a flexible neural model, the classical ICA formulation is no longer sufficient.

Although the present study is formulated as a BSS problem, its broader motivation is source-wise structured representation learning. In many scientific and engineering problems, the observed data are generated by multiple hidden factors that are entangled through an unknown observation mechanism. Recovering these factors is not only useful for signal separation, but also provides a concrete route toward decoupled and potentially identifiable latent representations. From this perspective, BSS serves as a verifiable testbed: if a model can assign different latent branches to different underlying source trajectories, then the same principle may later be extended to more general disentanglement or identifiable representation learning tasks, where each branch is expected to model one independent factor, mechanism, temporal mode, or structural degree of freedom. The goal of this paper is therefore not merely to separate a particular set of signals, but to develop a source-wise Transformer-based structural regularization framework in which each latent dimension is encouraged to acquire its own coherent and distinguishable structure.

The difficulty becomes particularly clear in nonlinear ICA. In contrast to the linear case, nonlinear ICA is generally not identifiable from independence alone; many nonlinear transformations can preserve independence while producing different latent coordinates \citep{hyvarinen1999nonlinear}. This observation is consistent with broader results in unsupervised disentanglement, where fully unsupervised recovery of ground-truth factors is impossible without appropriate inductive biases on the model or the data \citep{locatello2019challenging}. Consequently, modern identifiable representation learning has focused on identifying which additional structures can break the latent symmetry. Temporal nonstationarity, temporal dependence, auxiliary variables, grouped observations, and structured latent priors have all been shown to provide useful sources of identifiability or partial identifiability \citep{hyvarinen2016tcl,hyvarinen2017time,hyvarinen2019auxiliary,khemakhem2020ivae,halva2021snica,kivva2022identifiability}. These works suggest an important principle: source recovery or factor recovery becomes more plausible when the latent dimensions are not merely independent, but are equipped with distinguishable structural signatures.

Deep latent-variable models provide another route to flexible source representation. Variational autoencoders (VAEs) and related deep generative models enable nonlinear observation mappings and scalable variational learning \citep{kingma2014vae,rezende2014stochastic}. Variants such as $\beta$-VAE further demonstrate that modifying the balance between reconstruction and latent regularization can encourage more factorized or interpretable representations \citep{higgins2017betavae,burgess2018understanding}. However, standard VAE formulations typically use simple shared priors, such as an isotropic Gaussian prior, and therefore do not automatically assign different structural roles to different latent dimensions. Structured priors can address this limitation. Gaussian processes (GPs), for example, offer a principled way to encode temporal smoothness and covariance structure \citep{rasmussen2006gpml}, and GP-prior VAEs show that modeling correlations between samples can be beneficial for temporally or spatially structured data \citep{casale2018gppvae}. More generally, source-wise structured latent models have recently been explored by assigning each latent dimension its own adaptive prior, diffusion mechanism, or energy function \citep{wei2024halfvae,wei2026pdgmm,wei2026stradiff,wei2026strebm}. These approaches are aligned with the idea that separation, disentanglement, or identifiable latent organization can be promoted by making different latent dimensions structurally non-exchangeable.

At the same time, Transformers have become a powerful architecture for sequence and context modeling. Self-attention provides a flexible mechanism for aggregating information across a sequence \citep{vaswani2017attention}, while masked prediction has become a central self-supervised principle in language and vision models \citep{devlin2019bert,he2022mae}. In time-series analysis, Transformer-based models have been adapted to long-range forecasting, decomposition, frequency-domain modeling, and patch-based sequence representation \citep{zhou2021informer,wu2021autoformer,zhou2022fedformer,nie2023patchtst}. These developments show that patch tokenization, contextual reconstruction, and attention-based aggregation can capture rich temporal dependencies. Nevertheless, most Transformer time-series models are designed for prediction, imputation, classification, or representation learning on observed channels. They are not directly formulated as source-wise structural priors for BSS or disentangled latent modeling, where each latent dimension should evolve toward a distinct source component or explanatory factor while jointly reconstructing the observed mixtures.

This paper proposes StrTransformer, a source-wise structured Transformer framework for blind source recovery and, more broadly, for branch-wise structural latent modeling. The key idea is to treat each latent dimension as an individual source trajectory and to assign to each source its own structural Transformer branch. Unlike encoder-based latent-variable models, StrTransformer directly optimizes the latent source matrix together with an observation-space mixer and the source-wise Transformer parameters. The mixer enforces observation-space reconstruction, whereas the Transformer branches impose differentiable structural regularization on the individual source trajectories. Each branch evaluates the coherence of its assigned source through a masked patch reconstruction energy: patches are extracted from a one-dimensional source trajectory, randomly masked, embedded as tokens, processed by a locality-biased Transformer, and reconstructed from temporal context. Thus, the Transformer is not used as a direct generator of the source; rather, it functions as an implicit source-wise energy prior that scores whether a latent trajectory is structurally predictable from its own context.

A central component of StrTransformer is the ordered multi-scale controller. Instead of fixing the temporal scale of each source, the model learns ordered scale centers, soft patch-scale weights, and locality attention slopes. This design encourages lower-index branches to specialize in shorter-scale structures and higher-index branches to specialize in longer-scale structures, while still allowing the exact scale of each branch to be learned from data. The ordered controller therefore serves two purposes. First, it provides a flexible multi-scale mechanism for modeling different source structures. Second, it reduces the permutation symmetry that commonly appears in latent-variable models, because the branches are no longer structurally exchangeable. This is important not only for BSS, but also for future extensions to disentangled representation learning, where different latent factors should ideally be assigned stable and non-exchangeable structural roles.

The proposed objective combines observation reconstruction, source-wise masked patch reconstruction, and modular auxiliary regularizers. The reconstruction term ensures that the optimized sources jointly explain the observed sequence. The source-wise Transformer term encourages each branch to form a coherent temporal structure. Additional regularizers, such as source decorrelation, smoothness, scale entropy, and ordered scale-gap penalties, are treated as adjustable auxiliary terms rather than mandatory components. They can be strengthened, weakened, or deactivated depending on the data and the desired separation bias. This gives StrTransformer the form of a flexible objective template whose core is reconstruction plus source-wise structural specialization.

The main contributions of this paper are as follows. First, we propose a source-wise Transformer regularization framework for blind source recovery, where each latent dimension is constrained by its own masked patch reconstruction branch. Second, we introduce an ordered multi-scale controller that learns branch-specific patch scales and locality slopes, encouraging different latent dimensions to specialize in different temporal structures. Third, we formulate the structural Transformer loss as an implicit energy prior and combine it with a reconstruction-based mixer in an encoder-free optimization framework. Finally, we analyze the objective structure and validate the method through a controlled case study showing branch specialization and source-aligned recovery.
% ============================================================
% Methodology: StrTransformer
% ============================================================

\section{Methodology: StrTransformer}

\subsection{Source-wise structured latent formulation}

Let
\[
    Y=[y_1,\ldots,y_T]^\top\in\mathbb{R}^{T\times m},
    \qquad
    y_t\in\mathbb{R}^m ,
\]
be the observed multivariate sequence. StrTransformer introduces a latent source matrix
\[
    S=[s^{(1)},\ldots,s^{(K)}]\in\mathbb{R}^{T\times K},
    \qquad
    s_t=(s_{t1},\ldots,s_{tK})^\top\in\mathbb{R}^K ,
\]
where the column
\[
    s^{(k)}=(s_{1k},\ldots,s_{Tk})^\top\in\mathbb{R}^T
\]
is interpreted as the trajectory of the $k$-th source. Unlike an encoder-based VAE, StrTransformer does not introduce an inference map from $Y$ to $S$. The source matrix $S$ is directly optimized together with the mixer and the source-wise structural Transformer parameters.

The central principle is source-wise structural specialization. Instead of assigning one global temporal model to the whole latent vector, StrTransformer assigns one structural Transformer branch to each latent source. The $k$-th branch models the temporal or correlation structure of $s^{(k)}$. Different branches are therefore encouraged to specialize to different source structures, while the observation-space mixer enforces that the recovered sources jointly explain the measurements.

The row-wise observation model is
\begin{equation}
    \widehat y_t
    =
    \mathcal{M}_\theta(\widetilde s_t),
    \qquad
    \widehat Y
    =
    \mathcal{M}_\theta(\widetilde S),
    \label{eq:str_mixer}
\end{equation}
where $\mathcal{M}_\theta:\mathbb{R}^K\to\mathbb{R}^m$ is a time-shared observation-space mixer. The mixer may be chosen as an affine or nonlinear row-wise map depending on the experiment, but its specific architecture is not central to the StrTransformer formulation. Its role is to impose observation-space consistency, whereas the source-wise Transformer branches impose structured constraints in the latent source space.

The decoder input is
\begin{equation}
    \widetilde S
    =
    \Phi_{\mathrm{dec}}(S)
    =
    \begin{cases}
    S, & \text{without decoder-side source standardization},\\[1mm]
    \mathrm{Std}_{\mathrm{col}}(S), & \text{with decoder-side source standardization}.
    \end{cases}
    \label{eq:decoder_input_standardization}
\end{equation}
In the default implementation, $\widetilde S=S$.

Thus StrTransformer is not tied to a specific mixing family. The mixer provides the reconstruction path from latent sources to observations, while the source-wise Transformers provide differentiable structural regularization on the individual source trajectories. This separation is important: the mixer determines whether the optimized sources can jointly explain the observed measurements, whereas the Transformer branches evaluate whether each source trajectory exhibits a coherent source-wise temporal structure.

We use a quadratic reconstruction term
\begin{equation}
    \mathcal{L}_{\mathrm{rec}}(S,\theta)
    =
    \frac{1}{2\nu_y}
    \|Y-\mathcal{M}_\theta(\widetilde S)\|_F^2 .
    \label{eq:str_rec}
\end{equation}
The coefficient $\nu_y>0$ can be interpreted as a Gaussian observation-variance coefficient, but in the implemented objective it is treated primarily as a tunable hyperparameter controlling the relative weight of the reconstruction term. A smaller $\nu_y$ enforces stronger observation reconstruction, while a larger $\nu_y$ gives relatively more influence to the source-wise structural regularizers and auxiliary penalties. Therefore, $\nu_y$ should be understood as a reconstruction-fidelity coefficient rather than necessarily as a fixed physical noise variance.

\subsection{Source-wise structural regularization as an implicit energy prior}

For source $k$, let
\[
    \omega_k
    =
    \left(
        c_k,\alpha_k,\{\pi_{k,r}\}_{r=1}^{R},\{\mathcal{T}_{k,r}\}_{r=1}^{R}
    \right)
\]
denote its branch descriptor, including the ordered temporal-scale center $c_k$, the locality slope $\alpha_k$, the soft scale weights $\pi_{k,r}$, and the patch-scale Transformer modules $\mathcal{T}_{k,r}$. StrTransformer defines a source-wise structural energy
\[
    \mathcal{E}_k(s^{(k)};\omega_k),
\]
which measures how well the trajectory $s^{(k)}$ can be explained by its own temporal context under the $k$-th structural branch.

In the implemented objective, this energy is used as a regularization term rather than as an explicitly normalized probability density. Nevertheless, it admits an energy-based interpretation. Formally, one may associate the structural regularizer with an unnormalized Gibbs-type density
\[
    p_{\Psi,\eta}^{\mathrm{imp}}(S)
    \propto
    \exp\!\left[
        -
        \lambda_{\mathrm{str}}
        \mathcal{L}_{\mathrm{str}}(S;\Psi,\eta)
        -
        \mathcal{R}_{\mathrm{aux}}(S,\eta)
    \right],
\]
where the normalizing constant is not evaluated and is not needed for optimization. Therefore, the Transformer branch should be understood as an implicit energy-based structural prior, implemented through a differentiable source-wise regularizer.

The resulting optimization problem is
\begin{equation}
    \min_{S,\theta,\Psi,\eta}
    \quad
    \mathcal{J}(S,\theta,\Psi,\eta)
    =
    \mathcal{L}_{\mathrm{rec}}(S,\theta)
    +
    \lambda_{\mathrm{str}}\mathcal{L}_{\mathrm{str}}(S;\Psi,\eta)
    +
    \mathcal{R}_{\mathrm{aux}}(S,\eta),
    \label{eq:map_style_objective_general}
\end{equation}
where $\mathcal{L}_{\mathrm{rec}}$ is the observation reconstruction loss, $\mathcal{L}_{\mathrm{str}}$ is the source-wise multi-scale Transformer structural loss, and $\mathcal{R}_{\mathrm{aux}}$ contains the active auxiliary regularizers: source separation, smoothness, scale entropy, and ordered scale-gap penalties.

This formulation separates two roles. The mixer imposes the observation equation
\[
    \widehat Y=\mathcal{M}_\theta(\widetilde S),
\]
whereas the source-wise structural regularizer selects, among the decompositions that can explain the observations, the one whose individual source trajectories have the most coherent ordered temporal structures.

\subsection{Patch extraction as a linear operator}

Let
\[
    \mathcal{P}=\{P_1,\ldots,P_R\}
\]
be the set of candidate patch sizes. For a patch size $P_r$, the stride is
\begin{equation}
    q_r
    =
    \max\left\{
        1,
        \left\lfloor \rho P_r+\frac{1}{2}\right\rfloor
    \right\},
    \qquad
    0<\rho\leq 1 .
    \label{eq:patch_stride}
\end{equation}
Let $N_r$ be the number of extracted patches. For the $i$-th patch, define the binary extraction matrix
\[
    E_{r,i}\in\{0,1\}^{P_r\times T},
\]
so that
\[
    u_{k,r,i}
    =
    E_{r,i}s^{(k)}
    \in\mathbb{R}^{P_r}.
\]
Stacking all patches gives
\begin{equation}
    U_{k,r}
    =
    \Pi_r s^{(k)}
    =
    \begin{bmatrix}
    u_{k,r,1}^\top\\
    \vdots\\
    u_{k,r,N_r}^\top
    \end{bmatrix}
    \in\mathbb{R}^{N_r\times P_r}.
    \label{eq:patch_operator}
\end{equation}

Each patch is embedded as a Transformer token:
\begin{equation}
    x_{k,r,i}
    =
    W^{\mathrm{in}}_{k,r}u_{k,r,i}
    +
    b^{\mathrm{in}}_{k,r}
    +
    p_i,
    \label{eq:patch_embedding}
\end{equation}
where $p_i$ is a positional encoding. Thus, for each source and each candidate scale, the one-dimensional source trajectory is converted into a sequence of patch tokens.

\subsection{Locality-biased source-wise Transformer}

For source $k$ and scale $r$, let $\mathcal{T}_{k,r}$ be the corresponding Transformer encoder. In one attention head, the pre-softmax attention logit is written as
\begin{equation}
    \ell^{(k,h)}_{ij}
    =
    \frac{
        \langle q^{(h)}_{k,r,i},\,\kappa^{(h)}_{k,r,j}\rangle
    }{\sqrt{d_h}}
    -
    \alpha_k |i-j|
    +
    B^{\mathrm{mask}}_{ij},
    \label{eq:attention_logit}
\end{equation}
where $\alpha_k>0$ is the locality slope and $B^{\mathrm{mask}}_{ij}$ may encode causal or validity constraints. The attention weight is
\[
    A^{(k,h)}_{ij}
    =
    \frac{\exp(\ell^{(k,h)}_{ij})}
         {\sum_{j'}\exp(\ell^{(k,h)}_{ij'})}.
\]
A larger $\alpha_k$ penalizes distant patches more strongly and therefore induces a more local branch. A smaller $\alpha_k$ permits wider contextual aggregation. This provides a Transformer-based analogue of a temporal length-scale: short-scale branches use stronger locality bias, while long-scale branches use weaker locality bias.

\subsection{Masked patch reconstruction energy}

For source $k$ and patch scale $r$, let
\[
    U_{k,r}=\Pi_r s^{(k)}
\]
be the extracted patch sequence. Let $M\subset\{1,\ldots,N_r\}$ be a random mask set, with
\begin{equation}
    |M|
    =
    \max\left\{
        1,
        \left\lfloor \rho_{\mathrm{mask}}N_r+\frac{1}{2}\right\rfloor
    \right\}.
    \label{eq:mask_number}
\end{equation}
For $i\in M$, the patch token is replaced by a learned mask token. The masked token sequence is denoted by $X^{(M)}_{k,r}$, and the Transformer output is
\[
    H^{(M)}_{k,r}
    =
    \mathcal{T}_{k,r}
    \big(
        X^{(M)}_{k,r}; \alpha_k
    \big).
\]
The predicted patch is
\[
    \widehat u_{k,r,i}
    =
    W^{\mathrm{out}}_{k,r}h^{(M)}_{k,r,i}
    +
    b^{\mathrm{out}}_{k,r}.
\]

The source-scale structural energy is defined as the masked patch reconstruction error
\begin{equation}
    \ell_{k,r}(s^{(k)})
    =
    \mathbb{E}_M
    \left[
        \frac{1}{|M|P_r}
        \sum_{i\in M}
        \left\|
            \widehat u_{k,r,i}
            -
            u_{k,r,i}
        \right\|_2^2
    \right].
    \label{eq:masked_energy}
\end{equation}
In implementation, the expectation over masks is approximated by one randomly sampled mask set at each optimization step. Thus the Transformer branch is not used to generate the source directly; instead, it assigns a structural energy to the current source trajectory by measuring how well masked patches can be reconstructed from their temporal context.

A trajectory with coherent temporal structure at the selected scale tends to yield low masked reconstruction energy, because missing patches can be inferred from neighboring or long-range context. A trajectory that mixes incompatible structures tends to yield higher masked reconstruction energy. Therefore, \eqref{eq:masked_energy} acts as a self-supervised structural constraint on each source.

\subsection{Ordered multi-scale controller}

Different sources may exhibit different temporal scales or correlation structures. StrTransformer therefore gives each source branch a learnable ordered scale descriptor. Let
\[
    a_r=\log P_r,
    \qquad
    a_{\min}=\min_r a_r,
    \qquad
    a_{\max}=\max_r a_r .
\]
Introduce unconstrained parameters $\eta_0,\ldots,\eta_K$ and positive gaps
\[
    \delta_j=\mathrm{softplus}(\eta_j)+\varepsilon,
    \qquad
    j=0,\ldots,K,
    \qquad
    \varepsilon>0 .
\]
Define
\begin{equation}
    u_k
    =
    \frac{\sum_{j=0}^{k-1}\delta_j}
         {\sum_{j=0}^{K}\delta_j},
    \qquad
    c_k
    =
    a_{\min}
    +
    (a_{\max}-a_{\min})u_k .
    \label{eq:ordered_scale_center}
\end{equation}
Since all gaps are positive, the centers satisfy
\[
    a_{\min}<c_1<c_2<\cdots<c_K<a_{\max}.
\]
Therefore, the branch order is fixed: lower-index branches are assigned smaller temporal scales and higher-index branches are assigned larger temporal scales. The exact scale of each branch is not fixed, but is learned.

The branch-scale distribution is
\begin{equation}
    \pi_{k,r}
    =
    \frac{
        \exp\{-\tau(a_r-c_k)^2\}
    }{
        \sum_{\rho=1}^{R}
        \exp\{-\tau(a_\rho-c_k)^2\}
    },
    \qquad
    \tau>0 .
    \label{eq:scale_softmax}
\end{equation}
Here $\pi_{k,r}$ is a soft scale-selection weight rather than a prior probability of the source trajectory. It determines how strongly source branch $k$ uses the Transformer structural energy at patch size $P_r$.

The expected log-scale and expected patch size are
\begin{equation}
    \bar a_k
    =
    \sum_{r=1}^{R}\pi_{k,r}a_r,
    \qquad
    \bar P_k
    =
    \exp(\bar a_k).
    \label{eq:expected_scale}
\end{equation}
The locality slope is coupled to the same ordered coordinate:
\begin{equation}
    \alpha_k
    =
    \exp\!\left[
        \log\alpha_{\max}
        +
        (\log\alpha_{\min}-\log\alpha_{\max})u_k
    \right],
    \qquad
    0<\alpha_{\min}<\alpha_{\max}.
    \label{eq:locality_from_order}
\end{equation}
Consequently,
\[
    \alpha_1>\alpha_2>\cdots>\alpha_K .
\]
Small-scale branches therefore have stronger local attention bias, while large-scale branches have weaker local bias and can aggregate wider context.

The multi-scale structural loss is
\begin{equation}
    \mathcal{L}_{\mathrm{str}}(S;\Psi,\eta)
    =
    \frac{1}{K}
    \sum_{k=1}^{K}
    \sum_{r=1}^{R}
    \pi_{k,r}\,
    \ell_{k,r}(s^{(k)}).
    \label{eq:multiscale_structural_loss}
\end{equation}

To encourage scale specialization and prevent neighboring branches from collapsing to nearly identical roles, we use the scale entropy penalty
\begin{equation}
    \mathcal{L}_{\mathrm{ent}}
    =
    -\frac{1}{K}
    \sum_{k=1}^{K}
    \sum_{r=1}^{R}
    \pi_{k,r}\log(\pi_{k,r}+\varepsilon),
    \label{eq:entropy_loss}
\end{equation}
and the ordered gap penalty
\begin{equation}
    \mathcal{L}_{\mathrm{gap}}
    =
    \begin{cases}
    \displaystyle
    \frac{1}{K-1}
    \sum_{k=1}^{K-1}
    \left[
        \max\{0,\Delta_c-(c_{k+1}-c_k)\}
    \right]^2,
    & K>1,\\[4mm]
    0, & K=1 .
    \end{cases}
    \label{eq:gap_loss}
\end{equation}
Since $\mathcal{L}_{\mathrm{ent}}$ is minimized together with the total objective, it encourages sharper scale selection.

\subsection{Source separation and smoothness}

The reconstruction loss alone may admit many equivalent latent decompositions. StrTransformer therefore uses auxiliary penalties that promote source separation and temporal regularity.

Let
\[
    H_T
    =
    I_T-\frac{1}{T}\mathbf{1}\mathbf{1}^\top
\]
be the centering matrix. Define
\[
    V=H_TS,
    \qquad
    z_k
    =
    \frac{v_k}{\widehat\sigma_k+\varepsilon},
    \qquad
    Z=[z_1,\ldots,z_K],
\]
where $v_k$ is the $k$-th column of $V$ and $\widehat\sigma_k$ is its empirical standard deviation. The empirical correlation-like matrix is
\[
    C
    =
    \frac{1}{T}Z^\top Z .
\]
The source separation penalty is
\begin{equation}
    \mathcal{L}_{\mathrm{sep}}(S)
    =
    \|C-I_K\|_F^2 .
    \label{eq:sep_loss}
\end{equation}
Since the source columns are standardized before forming $C$, the diagonal part is nearly fixed, and the dominant effect of \eqref{eq:sep_loss} is to suppress cross-source correlations.

The smoothness penalty is
\begin{equation}
    \mathcal{L}_{\mathrm{smooth}}(S)
    =
    \frac{1}{(T-o)K}
    \|D_oS\|_F^2,
    \qquad
    o\in\{1,2\},
    \label{eq:smooth_loss}
\end{equation}
where $D_o$ is a first- or second-order difference operator.

\subsection{Full StrTransformer objective}

The active implemented objective is
\begin{equation}
\begin{aligned}
    \min_{S,\theta,\Psi,\eta}
    \quad
    \mathcal{J}
    =
    &
    \frac{1}{2\nu_y}
    \|Y-\mathcal{M}_\theta(\widetilde S)\|_F^2
    +
    \lambda_{\mathrm{str}}
    \frac{1}{K}
    \sum_{k=1}^{K}
    \sum_{r=1}^{R}
    \pi_{k,r}\ell_{k,r}(s^{(k)})
    \\
    &
    +
    \lambda_{\mathrm{sep}}\|C-I_K\|_F^2
    +
    \lambda_{\mathrm{smooth}}
    \frac{1}{(T-o)K}\|D_oS\|_F^2
    \\
    &
    +
    \lambda_{\mathrm{ent}}\mathcal{L}_{\mathrm{ent}}
    +
    \lambda_{\mathrm{gap}}\mathcal{L}_{\mathrm{gap}} .
\end{aligned}
\label{eq:full_strtransformer_objective}
\end{equation}

This objective matches the active training configuration used in the present implementation. The first term enforces observation reconstruction through the mixer. The second term is the source-wise ordered multi-scale Transformer structural regularizer, which is the main mechanism for encouraging each latent branch to match a coherent temporal structure.

The remaining terms should be understood as modular auxiliary regularizers rather than indispensable components that must always be used simultaneously. They all provide additional biases that can help promote source separation or structural specialization, but their necessity depends on the data, the mixer capacity, the source characteristics, and the desired level of regularization. In practice, each auxiliary term can be strengthened, weakened, or deactivated by tuning its corresponding coefficient. Setting the corresponding weight to zero removes that regularizer from the active objective. Therefore, \eqref{eq:full_strtransformer_objective} should be interpreted as a flexible objective template whose core consists of reconstruction and source-wise structural regularization, with additional separation-promoting terms selected according to the experimental setting.

\subsection{Decoupling and coupling analysis}

For fixed scale-controller parameters $\eta$, the structural loss decomposes over source columns:
\[
    \mathcal{L}_{\mathrm{str}}(S)
    =
    \frac{1}{K}\sum_{k=1}^{K}
    \mathcal{L}_{\mathrm{str}}^{(k)}(s^{(k)}),
    \qquad
    \mathcal{L}_{\mathrm{str}}^{(k)}(s^{(k)})
    =
    \sum_{r=1}^{R}\pi_{k,r}\ell_{k,r}(s^{(k)}).
\]
Therefore
\begin{equation}
    \nabla_S\mathcal{L}_{\mathrm{str}}
    =
    \frac{1}{K}
    \big[
        \nabla_{s^{(1)}}\mathcal{L}_{\mathrm{str}}^{(1)},
        \ldots,
        \nabla_{s^{(K)}}\mathcal{L}_{\mathrm{str}}^{(K)}
    ],
    \label{eq:block_gradient}
\end{equation}
and the Hessian with respect to source columns is block diagonal:
\begin{equation}
    \nabla^2_{S,S}\mathcal{L}_{\mathrm{str}}
    =
    \frac{1}{K}
    \operatorname{blockdiag}\bigl(
        \nabla^2_{s^{(1)}}\mathcal{L}_{\mathrm{str}}^{(1)},
        \ldots,
        \nabla^2_{s^{(K)}}\mathcal{L}_{\mathrm{str}}^{(K)}
    \bigr).
    \label{eq:block_hessian}
\end{equation}
Hence the structural Transformer branches are decoupled in the latent source coordinates.

The reconstruction loss is the main coupling term. Define
\[
    g_\theta(s_t)
    =
    \mathcal{M}_\theta(\phi_{\mathrm{dec}}(s_t)),
\]
where $\phi_{\mathrm{dec}}$ is the row-wise version of $\Phi_{\mathrm{dec}}$. Let
\[
    r_t
    =
    y_t-g_\theta(s_t),
    \qquad
    G_t
    =
    \frac{\partial g_\theta(s_t)}{\partial s_t}
    \in\mathbb{R}^{m\times K}.
\]
Then
\begin{equation}
    \nabla_{s_t}\mathcal{L}_{\mathrm{rec}}
    =
    -\frac{1}{\nu_y}G_t^\top r_t .
    \label{eq:rec_gradient_general}
\end{equation}
The Hessian is
\begin{equation}
    \nabla^2_{s_t}\mathcal{L}_{\mathrm{rec}}
    =
    \frac{1}{\nu_y}
    \left[
        G_t^\top G_t
        -
        \sum_{a=1}^{m}
        r_{t,a}
        \nabla^2_{s_t}g_{\theta,a}(s_t)
    \right].
    \label{eq:rec_hessian_general}
\end{equation}
At exact reconstruction, $r_t=0$, so
\begin{equation}
    \nabla^2_{s_t}\mathcal{L}_{\mathrm{rec}}
    =
    \frac{1}{\nu_y}G_t^\top G_t .
    \label{eq:rec_hessian_exact}
\end{equation}
Thus StrTransformer contains both a decoupled source-wise structural curvature and a reconstruction-induced coupling curvature. The former encourages each source to match its own temporal structure; the latter enforces joint explanation of the observations.

\subsection{Exact-reconstruction fiber and regularized source selection}

When the mixer is expressive, reconstruction alone may admit many source decompositions. Define the exact-reconstruction fiber
\begin{equation}
    \mathcal{F}_Y
    =
    \left\{
        (S,\theta):
        \mathcal{M}_\theta(\Phi_{\mathrm{dec}}(S))=Y
    \right\}.
    \label{eq:exact_fiber}
\end{equation}
Let the active non-reconstruction regularizer be
\[
    \mathcal{R}(S,\Psi,\eta)
    =
    \lambda_{\mathrm{str}}\mathcal{L}_{\mathrm{str}}
    +
    \lambda_{\mathrm{sep}}\mathcal{L}_{\mathrm{sep}}
    +
    \lambda_{\mathrm{smooth}}\mathcal{L}_{\mathrm{smooth}}
    +
    \lambda_{\mathrm{ent}}\mathcal{L}_{\mathrm{ent}}
    +
    \lambda_{\mathrm{gap}}\mathcal{L}_{\mathrm{gap}} .
\]
Then
\[
    \mathcal{J}
    =
    \frac{1}{2\nu_y}
    \|Y-\mathcal{M}_\theta(\Phi_{\mathrm{dec}}(S))\|_F^2
    +
    \mathcal{R}(S,\Psi,\eta).
\]

\begin{theorem}[Small-noise regularized fiber selection]
Assume $\mathcal{F}_Y\neq\emptyset$, $\mathcal{R}$ is bounded from below, and minimizers
\[
    (S_{\nu_y},\theta_{\nu_y},\Psi_{\nu_y},\eta_{\nu_y})
    \in
    \arg\min \mathcal{J}
\]
exist for a sequence $\nu_y\downarrow 0$. Then
\begin{equation}
    \|Y-\mathcal{M}_{\theta_{\nu_y}}(\Phi_{\mathrm{dec}}(S_{\nu_y}))\|_F^2
    =
    O(\nu_y).
    \label{eq:fiber_residual_rate}
\end{equation}
Moreover, every compact accumulation point of the minimizers lies in $\mathcal{F}_Y$ and minimizes $\mathcal{R}$ over the exact-reconstruction fiber.
\end{theorem}

\begin{proof}
Choose any exact pair $(S^\star,\theta^\star)\in\mathcal{F}_Y$ with compatible structural parameters. Since the residual of this pair is zero,
\[
    \mathcal{J}(S_{\nu_y},\theta_{\nu_y},\Psi_{\nu_y},\eta_{\nu_y})
    \leq
    \mathcal{R}(S^\star,\Psi^\star,\eta^\star).
\]
Because $\mathcal{R}$ is bounded from below by some constant $R_{\min}$,
\[
    \frac{1}{2\nu_y}
    \|Y-\mathcal{M}_{\theta_{\nu_y}}(\Phi_{\mathrm{dec}}(S_{\nu_y}))\|_F^2
    \leq
    \mathcal{R}(S^\star,\Psi^\star,\eta^\star)-R_{\min},
\]
which proves \eqref{eq:fiber_residual_rate}. Hence any compact accumulation point has zero reconstruction error and belongs to $\mathcal{F}_Y$.

For any $(S,\theta)\in\mathcal{F}_Y$,
\[
    \mathcal{J}(S_{\nu_y},\theta_{\nu_y},\Psi_{\nu_y},\eta_{\nu_y})
    \leq
    \mathcal{R}(S,\Psi,\eta).
\]
Taking the limit along a convergent subsequence and using lower semicontinuity of $\mathcal{R}$ gives that the accumulation point attains no larger regularizer value than any point in $\mathcal{F}_Y$. Thus it minimizes $\mathcal{R}$ on the exact-reconstruction fiber.
\end{proof}

This theorem clarifies the role of the StrTransformer regularizer. In the low-noise or high-fidelity regime, source recovery is determined by which exact decomposition has the lowest ordered source-wise structural energy.

\subsection{Permutation symmetry reduction}

Let $P$ be a $K\times K$ permutation matrix. Suppose the mixer class is closed under latent permutations, meaning that for any $\theta$ and $P$, there exists $\theta_P$ such that
\[
    \mathcal{M}_{\theta}(\Phi_{\mathrm{dec}}(S))
    =
    \mathcal{M}_{\theta_P}(\Phi_{\mathrm{dec}}(SP))
\]
for all relevant $S$. Then the reconstruction loss alone cannot distinguish $S$ from $SP$.

For branch $k$, define its structural descriptor as
\[
    \omega_k
    =
    \left(
        c_k,\alpha_k,\{\pi_{k,r}\}_{r=1}^{R},\{\mathcal{T}_{k,r}\}_{r=1}^{R}
    \right).
\]
The structural loss can be written as
\[
    \mathcal{L}_{\mathrm{str}}(S)
    =
    \frac{1}{K}
    \sum_{k=1}^{K}
    \mathcal{E}(s^{(k)};\omega_k).
\]

\begin{proposition}[Residual permutation group]
The subgroup of latent permutations preserving the ordered structural loss for all $S$ is
\begin{equation}
    \mathcal{G}_{\Omega}
    =
    \left\{
        P:
        \omega_{P(k)}=\omega_k
        \ \text{for all }k
    \right\}.
    \label{eq:residual_permutation_group}
\end{equation}
If the branch descriptors are pairwise distinct, then
\[
    \mathcal{G}_{\Omega}=\{I_K\}.
\]
\end{proposition}

\begin{proof}
After applying a permutation $P$, the source that was assigned to branch $P(k)$ is evaluated by branch $k$. The structural loss is invariant for all possible source matrices only if each permuted source is assigned to an identical branch descriptor. This gives the condition $\omega_{P(k)}=\omega_k$ for all $k$. If all descriptors are pairwise distinct, no non-identity permutation satisfies the condition.
\end{proof}

Therefore, the ordered controller both selects patch-scale preferences and reduces the exchangeability that appears in unconstrained latent-variable formulations.

\subsection{Linear ICA route through temporal joint diagonalization}

We now give an idealized linear analysis. Suppose the true observations are generated by
\begin{equation}
    y_t=A_\star x_t,
    \qquad
    A_\star\in\mathbb{R}^{m\times K},
    \qquad
    \mathrm{rank}(A_\star)=K,
    \label{eq:linear_true_model}
\end{equation}
where $x_t=(x_{t1},\ldots,x_{tK})^\top$ are zero-mean latent sources with
\[
    \frac{1}{T}X^\top X=I_K .
\]
Assume the observations are whitened and an exact affine reconstruction is achieved. Then every exact whitened solution can be written as
\begin{equation}
    S=XQ,
    \qquad
    Q^\top Q=I_K .
    \label{eq:whitened_orthogonal_ambiguity}
\end{equation}

For lag $\tau$, define
\[
    \Gamma_X(\tau)
    =
    \mathbb{E}[x_tx_{t+\tau}^\top]
    =
    \mathrm{diag}(\gamma_1(\tau),\ldots,\gamma_K).
\]
Let $\mathcal{T}_{\mathrm{lag}}$ be a finite lag set and define the temporal signature
\[
    g_k
    =
    (\gamma_k(\tau))_{\tau\in\mathcal{T}_{\mathrm{lag}}}
    \in\mathbb{R}^{|\mathcal{T}_{\mathrm{lag}}|}.
\]

\begin{assumption}[Simple temporal contrast]
There exists $\beta\in\mathbb{R}^{|\mathcal{T}_{\mathrm{lag}}|}$ such that
\[
    \lambda_k
    =
    \sum_{\tau\in\mathcal{T}_{\mathrm{lag}}}\beta_\tau\gamma_k(\tau)
\]
are pairwise distinct for $k=1,\ldots,K$.
\end{assumption}

\begin{theorem}[Weak linear recovery by temporal structure]
If $S=XQ$ and
\[
    Q^\top \Gamma_X(\tau)Q
\]
is diagonal for every $\tau\in\mathcal{T}_{\mathrm{lag}}$, then $Q$ is a signed permutation matrix. If, in addition, the ordered StrTransformer branch energies have a strict assignment margin for the true temporal signatures, then the remaining permutation is resolved by the ordered branches.
\end{theorem}

\begin{proof}
Define the contrast matrix
\[
    \Lambda
    =
    \sum_{\tau\in\mathcal{T}_{\mathrm{lag}}}\beta_\tau\Gamma_X(\tau)
    =
    \mathrm{diag}(\lambda_1,\ldots,\lambda_K).
\]
By assumption, $\Lambda$ has a simple spectrum. Since each $Q^\top\Gamma_X(\tau)Q$ is diagonal, their weighted sum
\[
    Q^\top \Lambda Q
\]
is also diagonal. This means that the columns of $Q$ are eigenvectors of $\Lambda$. Because $\Lambda$ is diagonal with distinct diagonal entries, its eigenvectors are the canonical basis vectors up to sign. Therefore $Q$ must be a signed permutation matrix.

Now define the branch-source structural cost matrix
\[
    H_{kj}
    =
    \sum_{r=1}^{R}
    \pi_{k,r}\,
    \ell_{k,r}(x^{(j)}).
\]
A strict assignment margin means that the identity assignment has lower total structural cost than any non-identity permutation:
\[
    \sum_{k=1}^{K}H_{kk}
    +
    \Delta
    \leq
    \sum_{k=1}^{K}H_{k,\pi(k)}
    \qquad
    \text{for all non-identity }\pi,
    \quad
    \Delta>0.
\]
Among signed permutations, sign ambiguities may remain if the losses are sign-symmetric, but nontrivial permutations have strictly larger ordered structural energy. Hence the ordered branches resolve the permutation ambiguity.
\end{proof}

\subsection{Nonlinear ICA route through conditional source structure}

For a nonlinear mixer, reconstruction alone is not enough. Assume
\begin{equation}
    y_t=f_\star(x_t),
    \qquad
    f_\star:\mathbb{R}^K\to\mathbb{R}^m
    \label{eq:nonlinear_true_model}
\end{equation}
is injective on the source support. If StrTransformer learns an injective mixer $\mathcal{M}_\theta$ and achieves exact reconstruction, then there exists an invertible map $h$ such that
\begin{equation}
    s_t=h(x_t).
    \label{eq:nonlinear_equivalent_representation}
\end{equation}
Thus the identifiability question becomes whether the structural objective rules out invertible maps $h$ that mix multiple source coordinates.

The masked context can be interpreted as a self-supervised auxiliary variable. Let $c_t$ denote the temporal context induced by the visible patches around time $t$. Assume the true source process satisfies a conditional factorization
\begin{equation}
    p(x_t\mid c_t)
    =
    \prod_{k=1}^{K}
    p_k(x_{tk}\mid c_t),
    \label{eq:conditional_factorization}
\end{equation}
and each conditional density belongs to a one-dimensional exponential family:
\begin{equation}
    p_k(x_{tk}\mid c_t)
    =
    h_k(x_{tk})
    \exp\left[
        T_k(x_{tk})^\top \lambda_k(c_t)
        -
        A_k(\lambda_k(c_t))
    \right].
    \label{eq:conditional_exponential_family}
\end{equation}

\begin{assumption}[Full-rank conditional modulation]
There exist contexts $c^{(0)},c^{(1)},\ldots,c^{(L)}$ such that the matrix formed by the differences
\[
    \lambda_k(c^{(\ell)})-\lambda_k(c^{(0)}
)
\]
has full rank over the collection of source-specific sufficient statistics.
\end{assumption}

\begin{theorem}[Conditional nonlinear recovery route]
Suppose $f_\star$ and the learned mixer are injective on their supports, exact reconstruction holds, the conditional factorization \eqref{eq:conditional_factorization}--\eqref{eq:conditional_exponential_family} is valid, and the full-rank modulation condition holds. If the learned representation also admits a factorized conditional source model with the same sufficient-statistic dimension, then
\begin{equation}
    s_t
    =
    \left(
        \phi_1(x_{t,\pi(1)}),
        \ldots,
        \phi_K(x_{t,\pi(K)})
    \right)
    \label{eq:nonlinear_componentwise_recovery}
\end{equation}
for component-wise invertible functions $\phi_k$ and a permutation $\pi$. If the ordered branch descriptors are pairwise distinct and satisfy a strict structural assignment margin, then $\pi$ is restricted to the branch-preserving permutation group in \eqref{eq:residual_permutation_group}; in the generic distinct-descriptor case, $\pi$ is the identity.
\end{theorem}

\begin{proof}
Exact reconstruction and injectivity imply that $s_t=h(x_t)$ for an invertible $h$. Consider conditional log-density ratios between a context $c$ and a reference context $c^{(0)}$:
\[
\begin{aligned}
    \log\frac{p(x_t\mid c)}{p(x_t\mid c^{(0)})}
    =
    \sum_{k=1}^{K}
    T_k(x_{tk})^\top
    \big[
        \lambda_k(c)-\lambda_k(c^{(0)})
    \big]
    -
    \sum_{k=1}^{K}
    \Delta A_k(c,c^{(0)}).
\end{aligned}
\]
Because $y_t=f_\star(x_t)$ and $s_t=h(x_t)$ are invertible reparameterizations on the support, the same conditional log-density ratio can also be written in terms of the learned coordinates $s_t$. The Jacobian terms cancel in the ratio because they do not depend on the context. Therefore the sufficient statistics of the true coordinates and the sufficient statistics of the learned coordinates are related by a linear transformation whose coefficient matrix is determined by the context-dependent natural parameters.

Under the full-rank modulation assumption, this linear relation is nondegenerate. Since the conditional model factorizes across coordinates, a learned coordinate that mixes two or more true coordinates would create cross-coordinate sufficient-statistic interactions, contradicting the factorized conditional representation except for component-wise invertible transformations and permutation. Hence \eqref{eq:nonlinear_componentwise_recovery} follows.

The remaining permutation is governed by the ordered structural descriptors. By Proposition \eqref{eq:residual_permutation_group}, only branch-preserving permutations remain. With pairwise distinct descriptors and a strict structural margin, the identity assignment is selected.
\end{proof}

\begin{remark}
This theorem should be interpreted as a conditional source-recovery route induced by the StrTransformer structural objective, rather than as a claim that arbitrary nonlinear ICA is identifiable from reconstruction alone. The additional information is supplied by source-wise temporal context, ordered branch descriptors, and masked contextual reconstruction. These ingredients make the latent coordinates non-exchangeable and give the objective a mechanism for selecting source-aligned decompositions from the exact-reconstruction fiber.
\end{remark}

\begin{algorithm}[t]
\DontPrintSemicolon
\caption{StrTransformer source optimization with ordered multi-scale structural regularizers}
\label{alg:strtransformer}

\KwData{
Observed sequence $Y\in\mathbb{R}^{T\times m}$; source number $K$; patch sizes $\{P_1,\ldots,P_R\}$; stride ratio $\rho$; mask ratio $\rho_{\mathrm{mask}}$; reconstruction variance $\nu_y$; loss weights $\lambda_{\mathrm{str}},\lambda_{\mathrm{sep}},\lambda_{\mathrm{smooth}},\lambda_{\mathrm{ent}},\lambda_{\mathrm{gap}}$; learning rate $\xi$; maximum iterations $N_{\max}$.
}

\KwResult{
Estimated sources $\widehat S$; trained mixer $\mathcal{M}_{\widehat{\theta}}$; trained structural Transformer parameters $\widehat{\Psi}$; ordered scale-controller parameters $\widehat{\eta}$.
}

Initialize source matrix $S$, mixer parameters $\theta$, source-wise Transformer parameters $\Psi$, and ordered-controller raw gaps $\eta$\;

\For{$n=1,\ldots,N_{\max}$}{
    Set $S_{\mathrm{raw}}\leftarrow S$\;
    Set $\widetilde S\leftarrow \Phi_{\mathrm{dec}}(S_{\mathrm{raw}})$ by \eqref{eq:decoder_input_standardization}\;
    Compute reconstruction $\widehat Y=\mathcal{M}_\theta(\widetilde S)$ and $\mathcal{L}_{\mathrm{rec}}=\|Y-\widehat Y\|_F^2/(2\nu_y)$\;

    Compute ordered centers, scale probabilities, expected patch sizes, and locality slopes
    $\{c_k,\pi_{k,1:R},\bar P_k,\alpha_k\}_{k=1}^{K}
    \leftarrow
    \mathrm{OrderedScaleController}_{\eta}(\{P_r\}_{r=1}^{R})$\;

    \For{$k=1,\ldots,K$}{
        \For{$r=1,\ldots,R$}{
            Set $q_r=\max\{1,\lfloor \rho P_r+1/2\rfloor\}$\;
            Extract $N_r$ patches $U_{k,r}=\Pi_{r,q_r}s^{(k)}_{\mathrm{raw}}$ with stride $q_r$\;
            Sample a masked patch set $M_{k,r}$ with cardinality $m_r=\max\{1,\lfloor \rho_{\mathrm{mask}}N_r+1/2\rfloor\}$\;
            Replace masked patch tokens by the learned mask token\;
            Run the source-scale Transformer with locality slope $\alpha_k$\;
            Compute $\ell_{k,r}(s^{(k)}_{\mathrm{raw}})$ by \eqref{eq:masked_energy}\;
        }

        Aggregate branch energy
        $\mathcal{L}^{(k)}_{\mathrm{str}}
        =
        \sum_{r=1}^{R}
        \pi_{k,r}\ell_{k,r}(s^{(k)}_{\mathrm{raw}})$\;
    }

    Compute
    $\mathcal{L}_{\mathrm{str}}
    =
    K^{-1}
    \sum_{k=1}^{K}
    \mathcal{L}^{(k)}_{\mathrm{str}}$\;

    Compute $\mathcal{L}_{\mathrm{sep}}$, $\mathcal{L}_{\mathrm{smooth}}$, $\mathcal{L}_{\mathrm{ent}}$, and $\mathcal{L}_{\mathrm{gap}}$ from $S_{\mathrm{raw}}$\;
    Form the active total objective $\mathcal{J}$ by \eqref{eq:full_strtransformer_objective}\;
    Update $(S,\theta,\Psi,\eta)\leftarrow(S,\theta,\Psi,\eta)-\xi\nabla_{S,\theta,\Psi,\eta}\mathcal{J}$\;
    Optionally clip gradients and apply the learning-rate scheduler\;
}

\textbf{return} $\widehat S=S$, $\widehat{\theta}=\theta$, $\widehat{\Psi}=\Psi$, $\widehat{\eta}=\eta$\;

\end{algorithm}

\section{Experimental Study}
\label{sec:case_study}

\subsection{Experimental setup and evaluation protocol}
\label{subsec:case_setup_eval}

We conduct a controlled case study to evaluate whether the proposed StrTransformer can recover source-aligned latent trajectories from mixed multivariate observations. The observed sequence is denoted by \(Y=[y_1,\ldots,y_T]^\top\in\mathbb{R}^{T\times m}\), and the model estimates a latent source matrix \(\widehat S=[\widehat s^{(1)},\ldots,\widehat s^{(K)}]\in\mathbb{R}^{T\times K}\). In the reported experiment, we set \(T=1000\) and \(K=3\). The three source branches are organized by the ordered multi-scale controller, so that lower-index branches are encouraged to represent shorter-scale structures and higher-index branches are encouraged to represent longer-scale structures. This ordering specifies only the structural role of each branch and does not provide the model with the ground-truth source identity.

The StrTransformer objective in \eqref{eq:full_strtransformer_objective} is optimized jointly over the latent source matrix, the observation-space mixer, the source-wise Transformer branches, and the ordered scale-controller parameters. During training, the model minimizes the reconstruction loss together with the source-wise masked patch reconstruction loss, source-separation penalty, smoothness penalty, scale-entropy penalty, and ordered scale-gap penalty. Reference sources are not used by the objective. They are used only after training to quantify source recovery and to visualize the aligned waveforms.

During training, we recorded the total objective, per-branch structural losses, mean absolute matched correlation, per-branch matched correlations, expected patch scales, ordered log-scale centers, locality attention slopes, and matched reference-source indices. These diagnostics are used to evaluate both numerical convergence and structural specialization of the learned branches.

Since source recovery is identifiable only up to permutation and sign in the post-hoc evaluation, the estimated sources are matched to the reference sources by absolute correlation. Let \(x^{(j)}\) denote the \(j\)-th reference source and \(\widehat s^{(k)}\) denote the \(k\)-th estimated source. We define \(\rho_{kj}=\left|\operatorname{corr}\left(\widehat s^{(k)},x^{(j)}\right)\right|\). The best permutation is obtained by \(\pi^\star=\arg\max_{\pi\in\mathfrak{S}_K}\sum_{k=1}^{K}\rho_{k,\pi(k)}\), and the mean absolute matched correlation is defined as \(\mathrm{MAC}=K^{-1}\sum_{k=1}^{K}\rho_{k,\pi^\star(k)}\). For waveform visualization, each estimated source is additionally sign-aligned with its matched reference source and then \(z\)-score normalized. This alignment is used only for evaluation and does not affect the learned model.

\subsection{Results and analysis}
\label{subsec:case_results_analysis}

Figure~\ref{fig:strtransformer_case_training} shows the training diagnostics. The total loss decreases rapidly during the early stage and then stabilizes at a low level, indicating that the joint optimization reaches a stable reconstruction-and-regularization balance. The per-branch structural losses also decrease substantially, although they exhibit small stochastic fluctuations. These fluctuations are expected because the structural energy is computed through randomly masked patch reconstruction at each optimization step.

The mean absolute matched correlation increases quickly and remains close to one after convergence. The per-branch matched correlation curves show that all three branches achieve high agreement with their matched reference sources, rather than the overall performance being dominated by only one branch. This suggests that the source-wise structural regularization does not merely improve the average representation quality, but helps all branches recover distinct source trajectories.

\begin{figure}[t]
    \centering
    \includegraphics[width=0.98\linewidth]{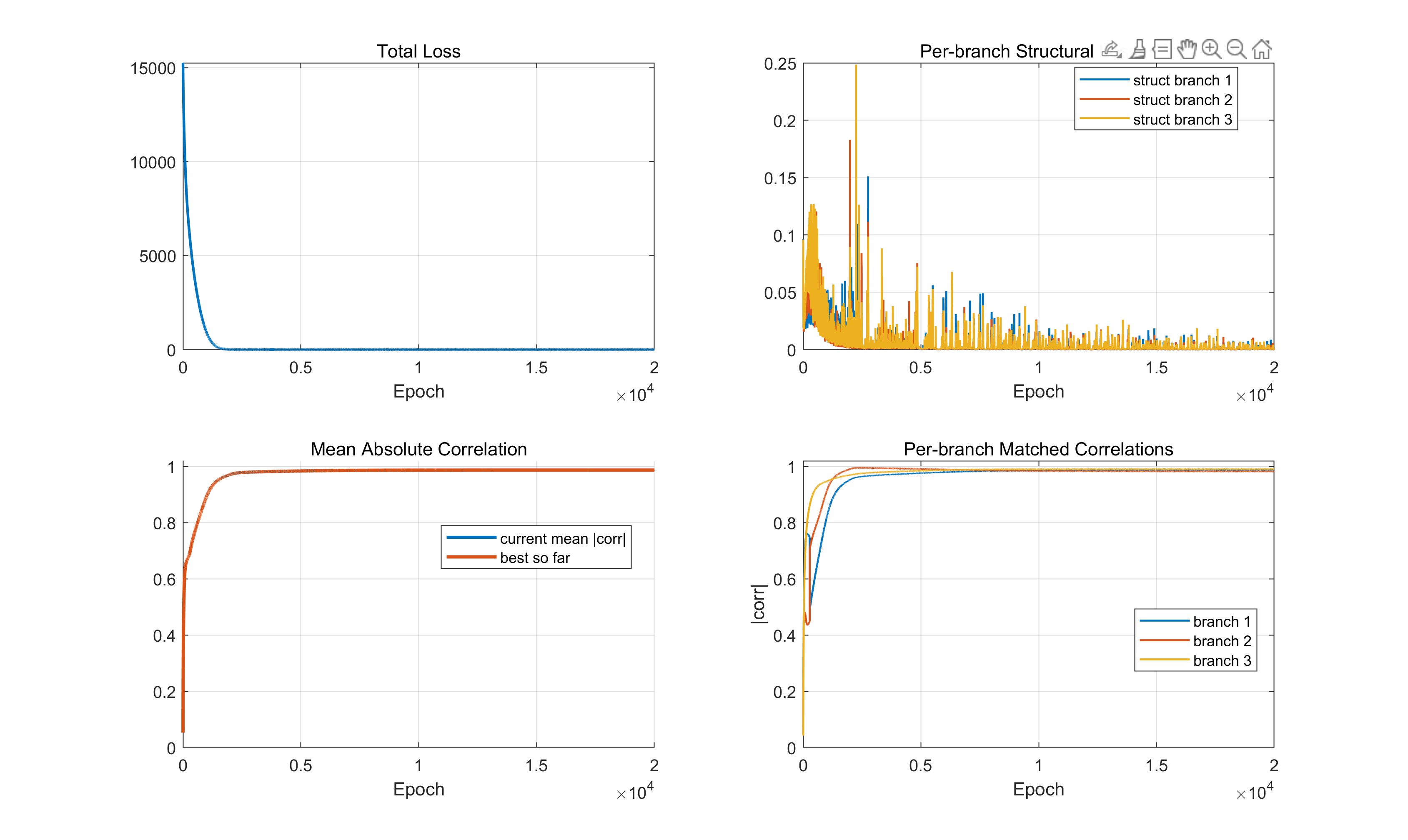}
    \caption{
    Training diagnostics of the ordered multi-scale StrTransformer.
    }
    \label{fig:strtransformer_case_training}
\end{figure}

Figure~\ref{fig:strtransformer_case_structure} reports the evolution of the ordered structural parameters. The expected patch scales separate into three clearly different regimes, corresponding to short-, intermediate-, and long-scale branches. The ordered log-scale centers remain separated throughout the later stage of training, showing that the scale-gap and entropy terms prevent the branches from collapsing to the same structural role.

The locality attention slopes exhibit the complementary behavior. The short-scale branch converges to a stronger locality bias, while the long-scale branch converges to a weaker locality bias and can therefore aggregate information over a wider context. Taken together, the expected patch scales, ordered log-scale centers, and locality slopes show that the three branches converge to different temporal structures. These diagnostics indicate that the ordered controller affects the learned branch structures rather than only assigning fixed branch labels.

The matched reference-source index becomes stable after a short transient period. This indicates that the learned branch assignment does not keep switching during optimization. Therefore, the ordered structural controller acts not only as a multi-scale selector, but also as a practical mechanism for reducing branch exchangeability.

\begin{figure}[t]
    \centering
    \includegraphics[width=0.98\linewidth]{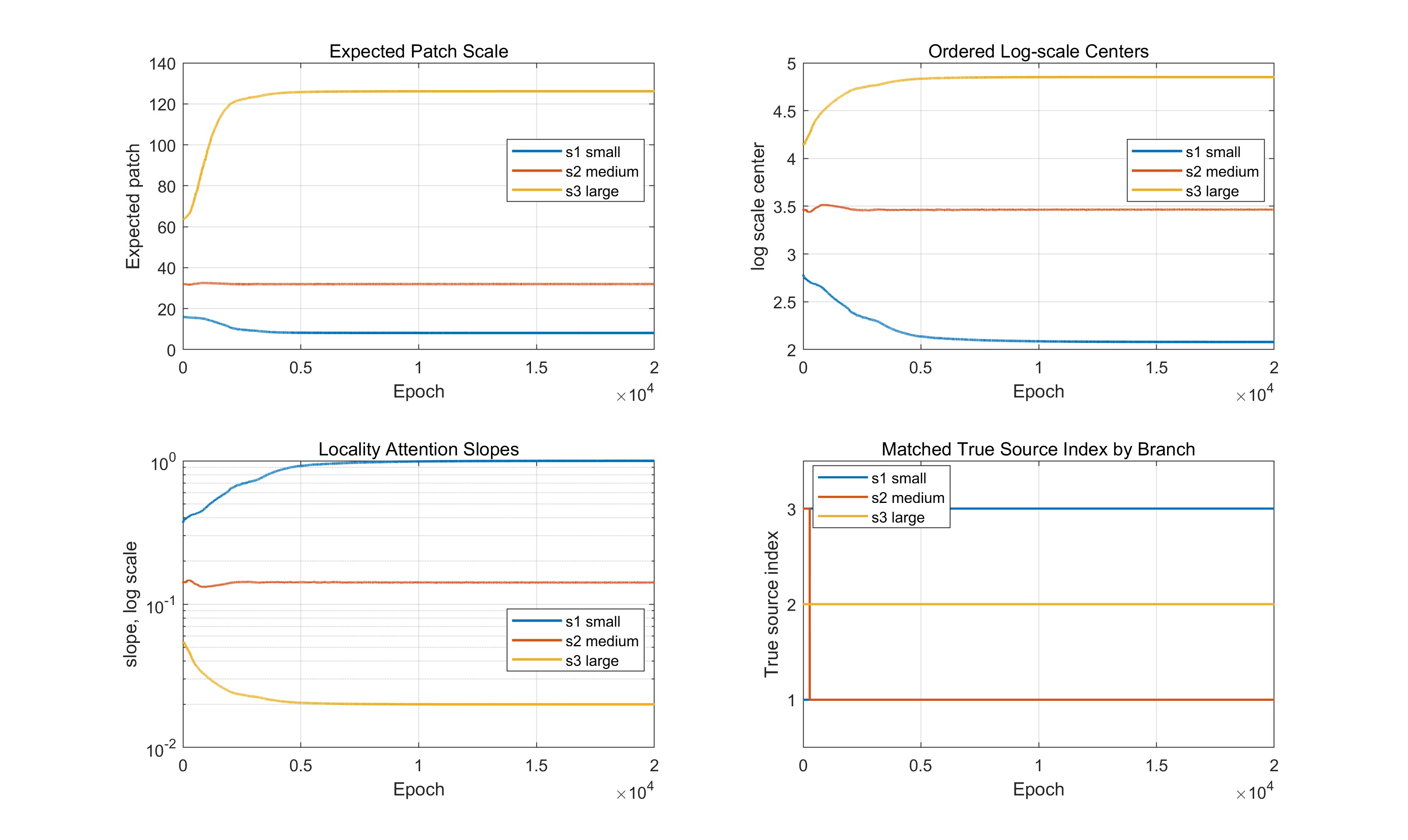}
    \caption{
    Learned ordered structural parameters.
    }
    \label{fig:strtransformer_case_structure}
\end{figure}

Figure~\ref{fig:strtransformer_case_sources} compares the recovered sources with the reference trajectories after post-hoc permutation, sign alignment, and $z$-score normalization. The estimated trajectories follow the reference signals across the full time interval and preserve the main temporal phase, amplitude-normalized shape, and component-wise structure. Small local deviations can still be observed near some peaks and transition regions, but the overall recovery remains strong.

It is worth noting that the synthetic signals in this case study are deliberately smooth. Such data are highly compatible with GP-kernel-related structural priors, because kernel length-scales provide a direct inductive bias for smooth temporal trajectories. By contrast, the Transformer branch used here learns structure through contextual masked patch reconstruction and locality-biased attention, which is more flexible but less directly matched to purely smooth GP-like signals. Consequently, the waveform-level fidelity in Fig.~\ref{fig:strtransformer_case_sources} is slightly lower than that typically obtained by GP-kernel-related source-wise methods such as StrADiff and StrEBM, where the structural priors are explicitly instantiated through GP priors or GP-inspired energies \citep{wei2026stradiff,wei2026strebm}. This behavior is consistent with the intended role of StrTransformer: it provides a more general source-wise contextual structural regularizer, while GP-kernel methods remain especially strong when the target sources are known to be smooth.

\begin{figure}[t]
    \centering
    \includegraphics[width=0.98\linewidth]{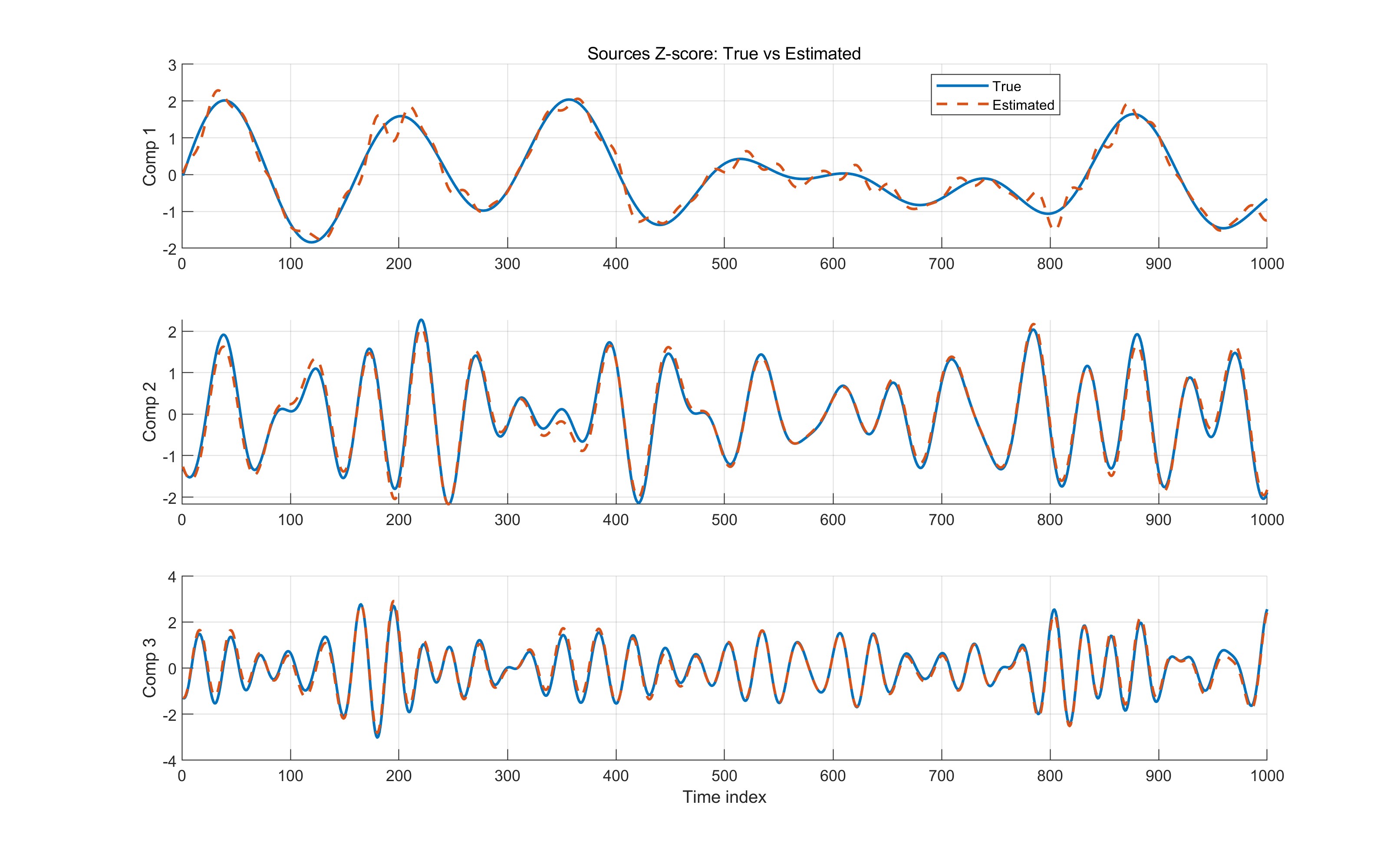}
    \caption{
    Z-score normalized source comparison after post-hoc permutation and sign alignment.
    }
    \label{fig:strtransformer_case_sources}
\end{figure}

Overall, the case study demonstrates three empirical properties of StrTransformer. First, the optimization is numerically stable: the total objective decreases rapidly and the correlation-based recovery metric remains high after convergence. Second, the ordered multi-scale controller learns distinct structural roles for different branches, as reflected by separated expected patch scales, ordered log-scale centers, differentiated locality slopes, and stable matched-source indices. Third, the recovered latent trajectories match the reference sources well after standard post-hoc alignment, although the smoothness-dominated artificial setting remains more naturally suited to GP-kernel-related structural priors.

The results are consistent with the intended design: the mixer maintains observation-space consistency, while the source-wise Transformer branches bias the latent variables toward trajectories that are coherent under their own patch-reconstruction energies. In this sense, StrTransformer uses reconstruction to maintain data fidelity and uses ordered source-wise structure to select a meaningful decomposition from the set of feasible latent representations.

\section{Conclusion}
\label{sec:conclusion}

This paper proposed StrTransformer, a source-wise structured Transformer framework for blind source recovery and branch-wise latent modeling. The central idea is to assign each latent source its own structural Transformer branch, so that different latent dimensions are encouraged to acquire distinct and coherent temporal structures while jointly reconstructing the observed sequence through an observation-space mixer. Instead of relying on an encoder, StrTransformer directly optimizes the latent source matrix together with the mixer, source-wise Transformer parameters, and ordered scale-controller parameters.

The proposed framework uses masked patch reconstruction as a differentiable source-wise structural energy. Together with the ordered multi-scale controller, this design encourages branch specialization across different temporal regimes and reduces the exchangeability of latent dimensions. The resulting objective combines reconstruction fidelity, source-wise structural regularization, and modular auxiliary penalties that can be adjusted according to the data and the desired separation bias. We also analyzed the decoupling and coupling structure of the objective, the role of the exact-reconstruction fiber, and the reduction of permutation symmetry induced by ordered branch descriptors.

A controlled case study showed that the learned branches converge to distinct temporal-scale structures and recover source-aligned latent trajectories under post-hoc evaluation. These results suggest that source-wise Transformer regularization can provide a useful route for blind source recovery, while also serving as a basis for broader decoupled and potentially identifiable latent representation learning. Future work should examine whether the same branch-wise structural mechanism remains effective for nonsmooth sources, nonstationary processes, and real multivariate measurements where the number and type of latent factors are not known in advance.
\clearpage
\bibliography{ref}

\end{document}